\documentclass[letterpaper, 10 pt, journal, twoside]{IEEEtran}

\usepackage{times,multirow,caption,float}
\usepackage{enumitem}
\usepackage{wrapfig}
\usepackage{graphicx}
\usepackage{epsfig,xspace,layout}
\usepackage{subfigure}
\usepackage{color}
\usepackage{amsfonts}
\usepackage{times}
\usepackage{amssymb}
\usepackage{amsmath}
\usepackage{rotating}
\usepackage{amsthm}
\usepackage{graphicx}
\usepackage{epsfig}
\usepackage{mathrsfs}
\usepackage{caption}
\usepackage{subfigure}
\usepackage{mathtools}
\usepackage{makeidx}
\usepackage{multirow} 
\usepackage{dblfloatfix}
\usepackage{threeparttable}
\usepackage{dsfont}
\usepackage[font={small}]{caption}
\usepackage{cleveref}
\usepackage{algorithm}
\usepackage{algorithmic}
\usepackage{tikz}
\usetikzlibrary{shapes,arrows,positioning,calc}
\usepackage{siunitx}
\theoremstyle{definition}
\newtheorem{theorem}{Theorem}

\newtheorem{proposition}{Proposition}

\newtheorem{definition}{Definition}
\newtheorem{problem}{Problem}
\theoremstyle{remark}
\newtheorem{remark}{Remark}

\DeclareMathAlphabet{\mathpzc}{OT1}{pzc}{m}{it}

\DeclareFontFamily{U}{jkpmia}{}
\DeclareFontShape{U}{jkpmia}{m}{it}{<->s*jkpmia}{}
\DeclareFontShape{U}{jkpmia}{bx}{it}{<->s*jkpbmia}{}
\DeclareMathAlphabet{\mathfrak}{U}{jkpmia}{m}{it}

\DeclareMathOperator{\arccot}{arccot}

\DeclareMathOperator{\dt}{dt}

\newcommand{\ys}{\mathsf{y}}
\newcommand{\xs}{x_{\mathtt{s}}}
\newcommand{\xsa}{x_{\mathtt{s},1}}
\newcommand{\xsb}{x_{\mathtt{s},2}}
\newcommand{\xsc}{x_{\mathtt{s},3}}
\newcommand{\xsd}{x_{\mathtt{s},4}}
\newcommand{\xf}{x_{\mathtt{f}}}
\newcommand{\xfa}{x_{\mathtt{f},1}}
\newcommand{\xfb}{x_{\mathtt{f},2}}
\newcommand{\xfc}{x_{\mathtt{f},3}}
\newcommand{\xfd}{x_{\mathtt{f},4}}
\newcommand{\xfe}{x_{\mathtt{f},5}}

\newcommand{\dxs}{\dot{x}_{\mathtt{s}}}
\newcommand{\dxf}{\dot{x}_{\mathtt{f}}}

\newcommand{\y}{\mathbf{y}}%

\newcommand{\0}{\mathbf{0}}
\newcommand{\R}{\mathbb{R}}

\renewcommand{\S}{\mathcal{S}}

\newcommand{\case}[1]{\begin{cases}#1\end{cases}}

\newcommand{\al}[1]{\begin{align}#1\end{align}}
\newcommand{\eq}[1]{\begin{equation}#1\end{equation}}
\newcommand{\ald}[1]{\begin{aligned}#1\end{aligned}}

\newcommand{\eqn}[1]{\begin{equation*}#1\end{equation*}}

\newcommand{\subeq}[1]{\begin{subequations}#1\end{subequations}}
\newcommand{\st}{\text{s.t. }}
\newcommand{\bmat}[1]{\begin{bmatrix} #1 \end{bmatrix}}

\title{Optimal Control of a Differentially Flat 2D Spring-Loaded Inverted Pendulum Model}
\author{Hua Chen$^{1}$, Patrick M. Wensing$^{2}$, and Wei Zhang$^{1}$%
\thanks{$^{1}$Hua Chen and Wei Zhang are with Department of Mechanical and Energy Engineering, Southern University of Science and Technology, Shenzhen, China.
{\tt\small chenh6@sustech.edu.cn, zhangw3@sustech.edu.cn}}%
\thanks{$^{2} $Patrick M. Wensing is with Department of Aerospace and Mechanical Engineering, University of Notre Dame, Notre Dame, IN 46556 USA
{\tt\small pwensing@nd.edu}}}
\begin{document}

\maketitle

\begin{abstract}
This paper considers the optimal control problem of an extended spring-loaded inverted pendulum (SLIP) model with two additional actuators for active leg length and hip torque modulation. These additional features arise naturally in practice, allowing for consideration of swing leg kinematics during flight and active control over stance dynamics. On the other hand, nonlinearity and the hybrid nature of the overall SLIP dynamics introduce challenges in the analysis and control of the model. In this paper, we first show that the stance dynamics of the considered SLIP model are differentially flat, which has a strong implication regarding controllability of the stance dynamics. Leveraging this powerful property, a tractable optimal control strategy is developed. This strategy enables online solution while also treating the hybrid nature of the SLIP dynamics. Together with the optimal control strategy, the extended SLIP model grants active disturbance rejection capability at any point during the gait. Performance of the proposed control strategy is demonstrated via numerical tests and shows significant advantage over existing methods.
	
\end{abstract}


\section{Introduction}\label{sec:introduction}
The Spring-Loaded Inverted Pendulum (SLIP) model has long played an important role in the development of robot locomotion~\cite{Geyer18}. Early biomechanics studies motivated the role of compliant leg operation~\cite{Blickhan89} for storing and releasing energy during running, while recent studies have likewise provided evidence for similar mechanisms in walking \cite{Geyer06}. From the physical embodiment of SLIP principles in Raibert's early hoppers~\cite{Raibert86} to compliant legged bipeds designed by Hurst and colleagues~\cite{Sreenath13, Hubicki18}, this promise of energetic economy through passive compliance has subsequently motivated many successful robot designs. Beyond design, the SLIP model has also served as a common template model~\cite{FullKoditschek99} to guide the control of hopping and running, both for robots that incorporate physical compliant mechanisms~\cite{Xin18,Tsagarakis17} and those that provide active compliance~\cite{De15} (e.g., via transparent actuation~\cite{Wensing17b}). 

These motivations from the design and control sides have led to extensive investigation into properties of the passive SLIP, and methods to actively control it. In 2D, the SLIP exhibits open-loop stable gaits~\cite{Holmes2006} for some paired combinations of its touchdown angle and leg stiffness. By comparison, lateral dynamics in 3D versions of the model disrupt this open-loop stability, requiring active stepping strategies~\cite{Seipel05}. Time-based swing leg strategies increase the robustness of SLIP models for operation on uneven terrain~\cite{Ernst10,Wu13,Liu16} and explain the robustness afforded by swing leg retraction~\cite{Seyfarth03}. Other active strategies have focused on coordinating touchdown angle selections with variations in the spring constant \cite{Saranli12,Wensing13b}, varying nominal leg length~\cite{Schmitt09,Byl16} or adjusting hip torque~\cite{Ankarali2010} during stance to modulate the total energy of the model. As a common thread, strategies are either applied open loop, or include once-per-step feedback. 

Gait stability for these methods has most commonly been studied using Poincar\'e analysis, which is complicated by the fact that the step-to-step return map does not admit an analytical solution. This property has motivated the development of approximate analytical solutions for the stance evolution of SLIP models \cite{Geyer05,Shahbazi15}, which may be used in the design of SLIP controllers \cite{Saranli12}. Another promising strategy, proposed by Piovan and Byl \cite{Byl16}, is to use partial feedback linearization techniques for leg-length modulation to analytically solve part of the dynamics. This previous work shows how active control can be used to simplify the SLIP dynamics, which is a motivating observation for the work herein.

Despite the fact that the extension of the classical SLIP with a leg length actuator or a hip torque actuator has been studied in the literature separately, the SLIP model containing both actuators has not been adequately investigated. In addition, swing leg evolution during flight is typically ignored, instead assuming instantaneous positioning of the leg. 
To address these limitations, this paper considers an extended SLIP model with both active leg length and hip torque modulation capabilities, and addresses swing leg kinematics during flight. During stance, this extended model is modeled as fully actuated and thereby differentially flat~\cite{Fliess1993,Fliess1995}. This property significantly simplifies analysis and control of the SLIP by working with its flat outputs~\cite{Martin2002,Martin2003}. While flatness-based planning and control strategies have demonstrated wide success in autonomous vehicles and quadrotors \cite{Villagra2007,Mellinger2011}, their application to legged robots has been less investigated (see e.g.,~\cite{Agrawal09}).  Part of the reason that flatness-based methods have not been more widely used is that they do not easily address constraints on control inputs. Within legged locomotion, constraints such as those on ground forces and actuators are a main challenge that limit system behavior, even in fully actuated regimes.

The contributions of this paper are as follows. First, a hybrid system model for an extended SLIP model is derived. This model features both leg length and hip torque modulations as well as a kinematic swing leg model. The kinematic swing leg model better captures limitations on achievable footholds given swing leg velocity limitations during flight. This combination of modeling features introduces challenges in the analysis and control due to nonlinearities of the stance dynamics and the hybrid nature of the overall dynamics. We show in this paper that the stance dynamics of the extended SLIP are {\em differentially flat}, which has a strong implication on controllability of the considered SLIP model. The flatness-based optimal control strategy for the extended SLIP is the {\em main contribution} of this paper. By exploiting the flatness property of the stance dynamics and other structures of the hybrid optimal control problem, a quadratic programming (QP) based control scheme is developed. The proposed scheme aims to match the behavior of the passive SLIP, which itself respects physical constraints on forces, lending a flatness-based strategy that respects constraints through soft penalization.  The resulting optimal control strategy is able to handle disturbances at any point during the gait, both in flight or in stance, which would be impractical for traditional Poincar\'e based once-per-step controllers. 

The paper is organized as follows. Section~\ref{sec:slipmodel} first derives the hybrid system model for the extended SLIP model considered in this paper, then details the differential flatness property of its stance dynamics. Section~\ref{sec:optcontrol} first formulates the optimal control problem of the hybrid SLIP dynamics, then describes the differential flatness-based solution strategy. This strategy enables on-line solution while also treating the hybrid dynamics of the model. Section~\ref{sec:simulation} provides demonstrations on the usage and performance of the proposed SLIP control strategy. Section~\ref{sec:concl} concludes the paper and discusses future work.

\section{Hybrid System Modeling of a Differentially Flat Spring-Loaded Inverted Pendulum}\label{sec:slipmodel}
In this paper, we consider an extended spring-loaded inverted pendulum (SLIP) model with two additional actuators. The first one is a linear actuator that can actively adjust the spring length and the second one is a rotary actuator that can inject hip torque. Similar to the passive SLIP model, evolution of this extended SLIP model contains flight phases and stance phases, as well as transitions between them. A graphical illustration of the considered SLIP is given in Fig.~\ref{fig:SLIP}. 

\begin{figure}[b] 
	\centering
		\includegraphics[width=0.8\linewidth]{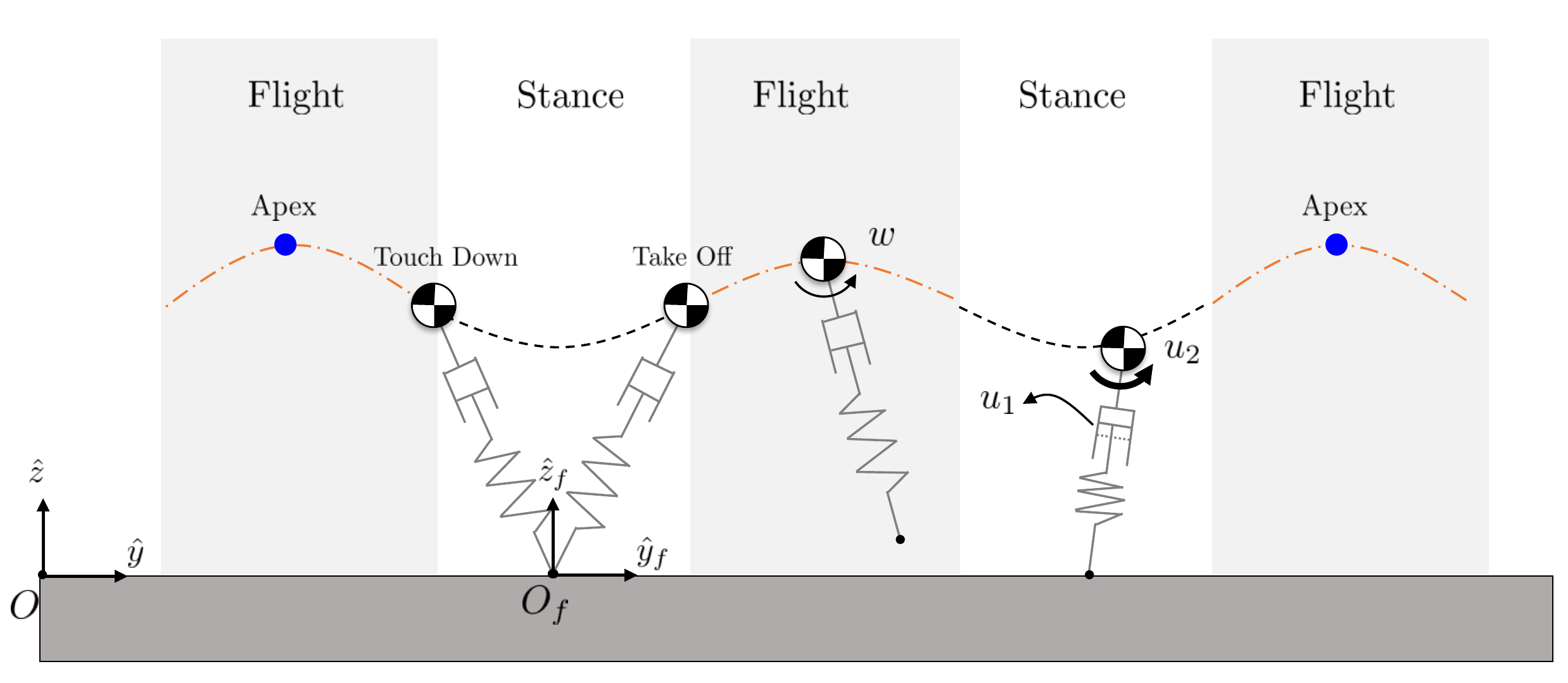} 
	\caption{\footnotesize Extended Spring-Loaded Inverted Pendulum }
	\label{fig:SLIP}
\end{figure}

Due to the additional actuators, the stance dynamics of the extended SLIP model are differentially flat, which significantly simplifies its analysis and control. In this section, we first derive a hybrid system model for the overall SLIP dynamics, then formally prove differential flatness of the stance dynamics. 



\subsection{Hybrid System Modeling of the Extended SLIP}
In addition to the traditional center of mass (CoM) dynamics, this paper considers the swing leg kinematics during flight phase as well. The swing leg kinematics are modeled considering swing leg angular speed as a control input. 
Let $(y,z)\in \R^2$ be Cartesian coordinates of CoM in a fixed world frame attached to the ground and let $\theta$ be the angle between the positive $y$ axis and swing leg measured counterclockwise. The state-space flight dynamics are then given by: 
\eqn{\ald{\ald{\dxf &= \bmat{ 0& 1 & 0 & 0 & 0 \\ 0 & 0 & 0 & 0 & 0 \\ 0 & 0 & 0 & 1 & 0 \\ 0 & 0 & 0 & 0 & 0  \\ 0 & 0 & 0 & 0 & 0 }\xf +\bmat{ 0\\0\\0\\0\\1  } w + \bmat{0\\0\\0\\-g\\0}  \\ & \triangleq f_\mathtt{f}(\xf,w), }}} where $\xf = (\xfa,\xfb,\xfc,\xfd,\xfe) = (y,\dot{y}, z, \dot{z},\theta)\in \R^5$ denotes the flight states, and $w$ is the swing leg angular speed.

\begin{remark}\label{rmk:swing}
The kinematic model of the swing leg during flight differs from the traditional instantaneous re-positioning in that we consider bounds on the angular speed $w$. These bounds result in changes to the reachable set of touch down angles over time. This important difference has a significant impact on the optimal control problem of the overall SLIP dynamics, as will be discussed in detail later in Section~\ref{sec:optcontrol}.
\end{remark}

In stance, let $\ell$ and $\theta$ be the leg length and leg angle, and let $u_1$ and $u_2$ be the displacement of the linear actuator and the torque generated by the rotary actuator respectively. Lagrangian techniques yield the following equations of motion:
\eqn{ \ald{m\ddot{\ell} & = m\ell\dot{\theta}^2-k(\ell-\ell_0)-mg\sin(\theta) + k u_1\\ m\ell^2\ddot{\theta}& = -mg\ell\cos(\theta) -2m\ell\dot{\ell}\dot{\theta} + u_2}}

Denoting the stance state by $x_{\mathtt{s}}= (x_{\mathtt{s},1},x_{\mathtt{s},2},x_{\mathtt{s},3},x_{\mathtt{s},4}) = (\theta,\dot{\theta},\ell,\dot{\ell})$, the resulting state-space stance dynamics are given by:
\eq{\label{eq:cslip_s} \ald{ \dxs   = &\bmat{ x_{\mathtt{s},2}\\ -2x_{\mathtt{s},2}x_{\mathtt{s},3}^{-1}x_{\mathtt{s},4}- g\cos(x_{\mathtt{s},1})x_{\mathtt{s},3}^{-1} \\ x_{\mathtt{s},4} \\ -g\sin(x_{\mathtt{s},1})+x_{\mathtt{s},2}^2x_{\mathtt{s},3}+\frac{k}{m}(\ell_0-x_{\mathtt{s},3})}\\ &\qquad \qquad \  \qquad\qquad  + \bmat{ 0 & 0 \\ 0 & \frac{1}{m} \xsc^{-2}\\ 0 & 0\\  \frac{k}{m} & 0} \bmat{ u_1 \\ u_2}\\  \triangleq &  f_{\mathtt{s}}(\xs)+ g_{\mathtt{s}}(\xs)u }}


Transitions between the stance and flight phases are governed by touch down (TD) and take off (TO) events as shown in Fig.~\ref{fig:SLIP}. All TD events lie on a four dimensional manifold given below:
\eqn{\S_\text{TD}  = \left\{\xf| \xfc - \ell_0 \sin(\xfe) = 0, \xfd <0. \right\},  }
while all TO events lie on a five dimensional manifold in the joint stance state-input space, which is given by\vspace{-2pt} 
\eqn{ \ald{\S_\text{TO}   = &\{ (\xs,u)|  \cos(\xsa)\xsb\xsc + \sin(\xsa)\xsd >0 ,\\ & \frac{k}{m}\sin(\xsa) (\ell_0 - \xsc + u_1)+\frac{\cos(\xsa)}{m\xsc} u_2 = 0 .\}.  }}
In the above definitions, $\S_\text{TD}$ is standard but $\S_\text{TO}$ adapts to the extended SLIP model considered in this paper, requiring horizontal acceleration to be zero and vertical velocity pointing upward. 

In addition, since state variables describing the flight and stance dynamics are different (but equivalent), the following changes of variables between them, denoted by $\Gamma_\mathtt{fs}$ and $\Gamma_\mathtt{sf}$, are needed.
\eqn{\ald{\Gamma_\mathtt{fs}:& \case{ \xsa & = \arccot(\frac{\xfa}{\xfc}) \\ \xsb & =  \frac{\xfa\xfd -\xfc\xfb}{(\xfa-y_\mathtt{f})^2+\xfc^2} \\ \xsc & = \sqrt{(\xfa-y_\mathtt{f})^2+\xfc^2} \\ \xsd & = \frac{\xfa \xfb + \xfc \xfd}{\sqrt{(\xfa-y_\mathtt{f})^2+\xfc^2} } }\\ \label{eq:s2f}\Gamma_\mathtt{sf}:& \case{ \xfa & = \cos(\xsa)\xsc+y_\mathtt{f} \\ \xfb & = -\sin(\xsa)\xsb\xsc + \cos(\xsa)\xsd \\ \xfc & = \sin(\xsa) \xsc \\ \xfd & = \cos(\xsa)\xsb\xsc + \sin(\xsa)\xsd \\ \xfe & = \xsa }}}

Putting all the above elements together, the overall dynamics of the proposed SLIP model is given below.
\eq{\label{eq:slip_hy} \ald{ & \dot{x}  = \case{  f_{\mathtt{s}}(x,\nu), & \text{ if }  \eta = 1,\\ f_{\mathtt{f}}(x,\nu), & \text{ if }  \eta = 0, } \\
		&  \nu = \case{u, & \text{ if } \eta = 1,\\ w, & \text{ if } \eta = 0, } \\ 
		&\eta^+  = \case{ 1, & \text{ if } x \in \S_\text{TD}, \\  0  & \text{ if } (x,\nu) \in \S_\text{TO}, \\ \eta, & \text{ otherwise, }} \\ & x^+  = \case{\Gamma_{\mathtt{fs}}(x) & \text{ if } x \in \S_\text{TD}, \\ \Gamma_{\mathtt{sf}}(x) & \text{ if } (x,\nu) \in \S_\text{TO}, \\ x, & \text{ otherwise. } } }}
Due to the difference between flight and stance dynamics and the transition events between them, the overall SLIP dynamics are essentially hybrid. Let $\mathbf{x} = (x,\eta)$ be the overall state variable containing both the continuous state and discrete mode, and let $\nu$ be the continuous control input, we write $\dot{\mathbf{x}} = f_\text{SLIP}(\mathbf{x},\nu)$ hereafter to abstractly denote the overall SLIP dynamics. Analysis and control of this SLIP dynamics is nontrivial, mainly due to the nonlinearities of the stance dynamics and the hybrid nature of the overall dynamics. 

To alleviate these issues, we show in the following subsection that the stance dynamics of the extended SLIP~\eqref{eq:cslip_s} is \emph{differentially flat}. This important feature offers a powerful tool for addressing the nonlinearities in the stance dynamics, which in turn relieves the difficulties in controlling the overall hybrid SLIP dynamics. 

%

\subsection{Differential Flatness of Stance Phase Dynamics}\label{sec:df_stance}

Differential flatness is a geometric property of general nonlinear control systems which was first introduced in~\cite{Fliess1993}. Roughly speaking, it extends the idea of controllability from linear systems to nonlinear systems. The formal definition of differential flatness is as follows.

\begin{definition}[\cite{Martin2002,Martin2003}]\label{def:df}
	A nonlinear control system $\dot{x} = f(x,u)$ with $x\in \R^n$ and $u\in \R^r$ is differentially flat, if there exist flat outputs $\mathsf{y}\in \R^r$ satisfying the following conditions.
	\begin{enumerate}[label=\arabic*.]
		\item There exists a (local) diffeomorphism $h$ such that,
		\eq{ \mathsf{y} = h(x,u,\dot{u},\ldots,u^{(p)})}
		\item There exist (local) diffeomorphisms $\phi$ and $\psi$ such that,
		\eq{\label{eq:invflat}\ald{ x & = \phi(\mathsf{y},\dot{\mathsf{y}},\ddot{\mathsf{y}},\ldots, \mathsf{y}^{(q)})\\ u & = \psi(\mathsf{y},\dot{\mathsf{y}},\ddot{\mathsf{y}},\ldots, \mathsf{y}^{(q+1)}) }}
		\item There does {\bf NOT} exist a function $\varphi$ such that \eq{\varphi(\mathsf{y},\dot{\mathsf{y}},\ddot{\mathsf{y}},\ldots, \mathsf{y}^{(s)}) = 0}
	\end{enumerate}
\end{definition}

\begin{remark}\label{rmk:df}
	The last condition is typically difficult to verify, even if the flat outputs and their derivatives are given. However, it has been shown in the literature that this condition will always hold if~\eqref{eq:invflat} holds~\cite{Fliess1993,Fliess1995,Waldherr2010}.
\end{remark}


\begin{theorem}\label{thm:df}
	The stance dynamics of the extended SLIP model~\eqref{eq:cslip_s} are \emph{Differentially Flat}.
\end{theorem}
\begin{proof}
We prove this theorem by checking the definition for a candidate set of flat outputs. For the given stance dynamics, we choose the flat outputs to be $\ys = (\ys_1,\ys_2) = (\xsc\cos(\xsa),\xsc\sin(\xsa))$. Differentiating $\ys_1$ and $\ys_2$ twice results in the following relationships.
\eq{\label{eq:flatoutput}\ald{ \ys_1 & = \cos(\xsa)\xsc \\ 
		\dot{\ys}_1 & = -\sin(\xsa)\xsb\xsc+\cos(\xsa)\xsd \\ 
		\ddot{\ys}_1 & = \frac{k}{m}\cos(\xsa) (\ell_0 - \xsc +u_1) - \frac{\sin(\xsa)}{m\xsc}u_2 \\ 
		\ys_2 & = \sin(\xsa)\xsc \\ 
		\dot{\ys}_2 & = \cos(\xsa)\xsb \xsc +\sin(\xsa)\xsd \\ 
		\ddot{\ys}_2 & = \frac{k}{m} \sin(\xsa)(\ell_0 -\xsc + u_1) + \frac{\cos(\xsa)}{m\xsc}u_2 - g 
	}}

Solving the above algebraic equations for $\xsa$, $\xsb$, $\xsc$, $\xsd$, $u_1$ and $u_2$ yields
\eq{\label{eq:phi}\ald{
		&\xsa  = \arccot(\frac{\ys_1}{\ys_2}),  & \xsb  = \frac{\ys_1\dot{\ys}_2 - \ys_2\dot{\ys}_1}{\ys_1^2+\ys_2^2} \\ 
		&\xsc  = \sqrt{\ys_1^2+\ys_2^2},   & \xsd  =  \frac{\ys_1\dot{\ys}_1 + \ys_2\dot{\ys}_2}{\ys_1^2+\ys_2^2}
	}}

\eq{\label{eq:psi}\ald{
		u_1 & =  \sqrt{\ys_1^2+\ys_2^2} + \frac{m\ys_1\ddot{\ys}_1+m\ys_2\ddot{\ys}_2+mg\ys_2}{k\sqrt{\ys_1^2+\ys_2^2}}-\ell_0, \\ 
		u_2 & = 	mg\ys_1+m\ys_1\ddot{\ys}_2-m\ys_2\ddot{\ys}_1
}}

These relationships verify the existence of diffeomorphisms $\phi$ and $\psi$, which in turn proves that $\ys$ satisfies all the conditions in Definition~\ref{def:df}. Therefore, the stance dynamics~\eqref{eq:cslip_s} is differentially flat.
\end{proof}

\begin{remark}\label{rmk:dfproof}
	An alternative proof of this theorem directly follows from the fact that any fully actuated holonomic system is differentially flat. In addition, any such system admits its configuration variables as a valid set of flat outputs~\cite{Martin2002}. Due to the fact that differential flatness is a geometric property, it is further independent of the choice of coordinates. Therefore, selections of flat outputs are not unique. The configuration variables $(\theta,\ell)$ serve as a valid set of flat outputs as well.
\end{remark}



Owing to the differential flatness property, any trajectory in the flat output space corresponds to a controlled trajectory of the original nonlinear dynamics and vice versa. This property enables us to consider analysis and control problems in the flat output space, in which complicated differential constraints become a simple chain of integrators. 

Albeit the fact that differential flatness addresses the nonlinearities in the stance dynamics, the overall SLIP dynamics remains hybrid. In the following section, we rigorously formulate the optimal control problem, then provide a tractable solution that fully exploits the differential flatness property.

\section{Optimal Control of the Differentially Flat SLIP Model}\label{sec:optcontrol}
Control of the extended SLIP aims to find both the swing leg angular speed during flight and the leg length and hip torque adjustments during stance to ensure stable gait. Without loss of generality, inspired by the periodic nature of the SLIP dynamics, we consider the hybrid optimal control problem of the extended SLIP model within one complete period between two consecutive take off events. The control objective is to achieve a certain desired take off state while respecting all constraints. 

\begin{problem}\label{prob:slip_plan}
	Given the SLIP dynamics~\eqref{eq:slip_hy}, a desired hybrid take off state $\mathbf{x}^\text{d} = (x^\text{d},1)$, for any given hybrid initial state $\mathbf{x}^0  = (x^0,\eta^0)$ with corresponding initial input $\nu^0$, find a solution to the following problem. 
	\eq{ \label{eq:slip_plan}\ald{  \min\limits_{T,\nu }\   & \int\limits_0^{T}  c(\mathbf{x},\nu)\dt  + c_\text{T}(\mathbf{x}(T))\\ \st \   & \dot{\mathbf{x}} = f_\text{SLIP}(\mathbf{x},\nu),  \\ & \mathbf{x}(0) = \mathbf{x}^0, \quad \nu(0) = \nu^0, \\ & \eta(T) = 1,\quad  \left(x(T),\nu(T)\right) \in \S_\text{TO},   }} where $c(\mathbf{x},\nu) = \|x - x^\text{ref}\|_Q^2+\|\nu - \nu^\text{ref}\|_R^2$ is the running cost which penalizes the deviation from some reference signal $ x^\text{ref}$ and $\nu^\text{ref}$, and  $c_T(\mathbf{x}(T)) = \|x(T) - x^\text{d}\|_{Q_T}^2$ is the terminal cost penalizing the deviation from the desired take off state $x^\text{d}$. $Q \succeq 0$, $R\succeq 0$, $Q_T\succ 0$ are weighting matrices. 
\end{problem}

As discussed in the above section, the hybrid nature of the extended SLIP dynamics and the nonlinearities in the stance phase dynamics are the main challenges, which considerably complicate the control problem. By taking advantage of the differential flatness property and an insightful observation on the structure of the hybrid dynamics, a tractable solution scheme to optimal control of the extended SLIP is developed. The key of such a solution scheme is to decompose the overall problem into a stance phase problem and a flight phase problem, according to the discrete mode of the initial state. A graphical illustration of the decomposition is provided in Fig.~\ref{fig:ocp}. 
\begin{figure}[t!] 
	\centering
	\includegraphics[width=0.8\linewidth]{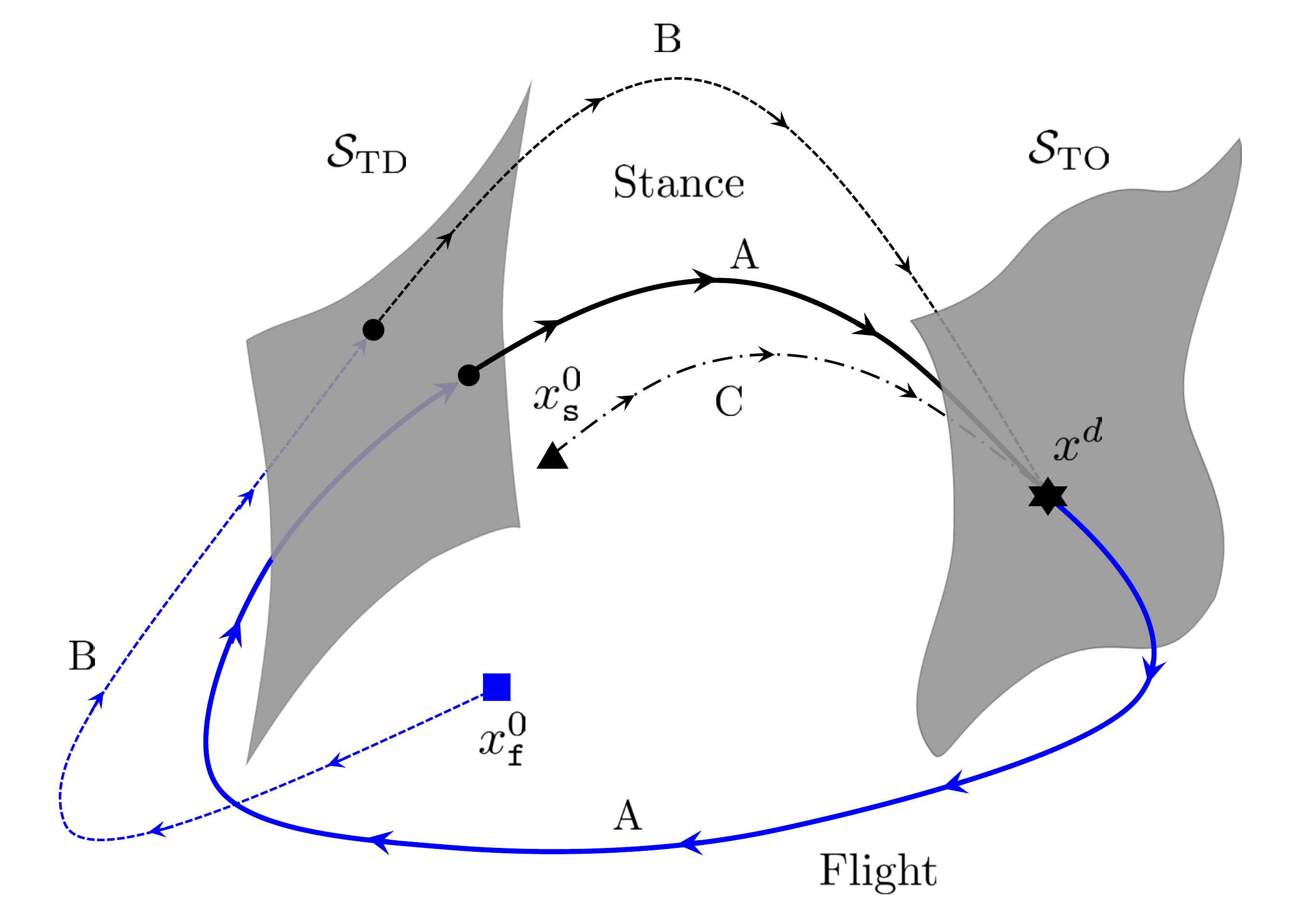} 	
	\caption{ \footnotesize Illustration of the optimal control problem. All flight phase trajectories are shown in blue and all stance phase trajectories are shown in black. Trajectory A (solid line) shows a periodic gait. Trajectory B (dotted line) is a solution with a flight phase initial state. Trajectory C (dash-dotted line) is a solution with a stance phase initial state. Arrows along trajectories indicate time flow direction.\label{fig:ocp} }
\end{figure}

\subsection{Stance Phase Optimal Control - Flatness-based Solution} \label{sec:oc_stance}
The stance phase optimal control problem is concerned with finding the optimal linear actuator displacement and hip torque signals for a given initial stance state. If the initialization is given as a stance state, i.e., $\eta^0 = 1$, Problem~\ref{prob:slip_plan} simply becomes
	\eq{ \label{eq:slip_plan_s}\ald{ \min\limits_{T_S,u}  & \ \int\limits_{0}^{T_S} c(\xs,u)\dt  + c_\text{T}(\xs(T_S))\\ \st \   & \dxs = f_\mathtt{s}(\xs,u),  \\ & \xs(0) = \xs^0, \quad u(0) = u^0,  \\ & \left( \xs(T_S),u(T_S)\right) \in \S_\text{TO} . }}

Since the desired goal state in the overall optimal control problem is a stance state, this stance phase problem is a standard optimal control of continuous nonlinear dynamics, involving no discrete variables. Differential flatness of the stance dynamics allows for a reformulation of~\eqref{eq:slip_plan_s} using flat outputs, resulting in a much simpler exposition of the optimal control problem that admits a tractable solution based on quadratic programming.


Let $\y = (\ys_1,\dot{\ys}_1,\ddot{\ys}_1, \ys_2,\dot{\ys}_2,\ddot{\ys}_2)$ be the flat outputs and their derivatives. The flat outputs $\y^0$, $\y^{\text{ref}}$, and $\y^d$ associated with any initial condition $(\xs^0,u^0)$, any reference trajectory $(\xs^{\text{ref}},u^{\text{ref}})$, and the desired terminal state $\xs^\text{d}$ can be easily computed via~\eqref{eq:flatoutput}. Reformulation of~\eqref{eq:slip_plan_s} in the output space is then given below. 	
\eq{\label{eq:stance_flat}\ald{
			\min\limits_{\y,T_S}  & \int\limits_0^{T_S} \| \y(t) - \y^{\text{ref}}(t)\|_{Q^\y_S}^2 \dt  + \|\y(T_S) - \y^d\|_{Q_T^\y}^2 \\
			\st  &  \y(0) = \y^0,  \\ 
			& \dot{\ys}_2(T_S)>0, \ \ddot{\ys}_2(T_S) = -g. 
}}
\begin{remark}\label{rmk:soft_penalization}
Finding the optimal time horizon $T_S$ is known to be hard for optimal control of nonlinear dynamics. In addition, potential state and\slash or input constraints, e.g., nonnegative ground reaction force and actuator limits, would drastically complicate the above optimization, mitigating the benefits from differential flatness. These issues are common to all flatness-based approaches. Existing flatness-based approaches (e.g.,~\cite{Mellinger2011}) heavily rely on heuristic initial guesses on the time horizon and numerical optimization techniques such as gradient descent. However, due to the natural response of the passive SLIP, arbitrarily selected $T_S$ and the associated reference trajectory $\xs^\text{ref}$ may not be physically meaningful.

In this paper, we exploit natural response of the passive SLIP model and develop a soft penalization scheme to tractably address all the above issues. To be specific, the key observation is that the passive SLIP response naturally satisfies the nonnegative ground reaction force constraint and all other actuator limit constraints, and directly returns an associated stance time. Hence, by selecting reference trajectory in the optimal control problem as the passive SLIP solution, stance control is regulated to generate trajectories that deviate as little as possible from the passive SLIP solution. 
\end{remark} 

Parametric function classes that are closed under differentiation have been widely used for solving the above infinite-dimensional optimization problem. In our solution, the following polynomial approximation is adopted.
\al{\label{eq:flat_poly}\mathsf{y}_1(t)  \approx \phi_0(t) \bar{\alpha},\quad \ys_2(t)&\approx  \phi_0(t) \bar{\beta},} where $\phi_0(t) = [1,\ t,\ t^2,\ \ldots,\ t^N]$ is the polynomial basis, $N$ is the maximum degree used in the parameterization, and $\bar{\alpha} = (\alpha_0,\alpha_1,\ldots,\alpha_N)$, $\bar{\beta} = (\beta_0,\beta_1,\ldots,\beta_N)$ are the coefficients to be determined. Let $\phi_1(t) = [0,\ 1,\ 2t,\ \ldots,\ Nt^{N-1}]$ and $\phi_2(t) = [0,\ 0,\ 2,\ \ldots,\ N(N-1)t^{N-2}]$, we have 
\subeq{\label{eq:flat_poly_deri}\al{ \dot{\ys}_1(t) & = \phi_1(t)\bar{\alpha},\quad \dot{\ys}_2(t)= \phi_1(t)\bar{\beta}, \\ \ddot{\mathsf{y}}_1(t) &=  \phi_2(t) \bar{\alpha},\quad  \ddot{\mathsf{y}}_2(t) =  \phi_2(t) \bar{\beta}.}}
Denote by $\Phi(t) = \bmat{ \phi^T_0(t) & \phi^T_1(t) & \phi^T_2(t)}^T$, $\bar{\Phi}(t) = \text{blkdiag}(\Phi(t),\Phi(t))$ for every $t$ and $\gamma = (\bar{\alpha}, \bar{\beta})$, the finite dimensional approximation to~\eqref{eq:stance_flat} with this polynomial parameterization is given by 

\eq{\ald{
		\min\limits_{\gamma }  & \int\limits_0^{T_S} \| \bar{\Phi}(t)\gamma - \y^{\text{ref}}(t)\|_{Q^\y_S}^2\dt  +  \|\bar{\Phi}(T_S)\gamma - \y^\text{d}\|_{Q^\y_T}^2 \\
		\st  & \bar{\Phi}(0) \gamma = \y^0,\\ 
		& \bmat{ 0 & \phi_1(T_S)}\gamma>0,\quad  \bmat{ 0 & \phi_2(T_S)}\gamma= -g .  
}}	
The above optimization is a standard quadratic programming (QP) in $\gamma = (\bar{\alpha}, \bar{\beta})$, which can be efficiently solved using various available solvers. Solution to the above QP gives parameters of the polynomial approximations. The corresponding flat outputs can then be computed via~\eqref{eq:flat_poly} and~\eqref{eq:flat_poly_deri}. Original states and inputs are in turn determined by~\eqref{eq:phi} and~\eqref{eq:psi}.

Apart from being important on its own, solution to the stance phase problem for all possible stance states constructs a value function for the stance dynamics. Such a value function serves as an important piece for solving the overall optimal control problem~\eqref{eq:slip_plan}. Henceforth, we denote by $V_\text{S}$ the value function for the stance phase problem. Utilizing this value function, we show in the following subsection how the optimal control problem can be solved when the initial state is given as a flight phase state.  


\subsection{Flight Phase Optimal Control}\label{sec:oc_flight}
Given a flight phase initial state, the hybrid nature of the SLIP dynamics comes into play. Relying on the aforementioned stance phase value function $V_\text{S}$ and the \emph{Bellman's principle of optimality}, the optimal control problem~\eqref{eq:slip_plan} can be reformulated as follows. 
\eq{\label{eq:f1}\ald{
			\min\limits_{w,T_F }  & \int\limits_0^{T_F} c(\xf,w) \dt + V_\text{S}\left(\Gamma_{\mathtt{fs}}(\xf(T_F))\right) &\\
			\st  & \dxf = f_\mathtt{f}(\xf,w),  \\
			& \underline{w} \le w \le \overline{w},   \\ 
			& \xf(T_F) \in \S_\text{TD}.
}}

Despite the simplicity of the flight phase dynamics,~\eqref{eq:f1} remains highly nontrivial to solve, mainly due to the four dimensional TD manifold that couples the flight phase and stance phase and the following implicit constraint between $w$ and $T_F$: 
\eq{\label{eq:Tf} \frac{1}{2} g T_F^2  - \xfd^0 T_F +\ell_0 \sin\!\left(\xfe^0 +\int\limits_0^{T_F} w(t) \dt \right) - \xfc^0 = 0 .} To design a tractable solution to the above optimal control problem, we first observe an important structural property of the touch down manifold, summarized the following proposition.
\begin{proposition}\label{prop:manifold}
Given any feasible initial state $\xf^0$ during flight, the corresponding TD manifold $\S_\text{TD}$ is only one dimensional and is exactly paramerterized by the touch down angle $\theta_f$. 
\end{proposition}

\begin{proof}
Define the following change of variable \eq{\theta_f = \xfe^0 +\int\limits_0^{T_F} w \dt, } where $\theta_f$ is the touch down angle. Plugging this change of variable into~\eqref{eq:Tf}, $T_F$ can be obtained by simply solving the quadratic equation in $T_F$, which yields 
\eqn{
T_F  = \frac{1}{g}\left(\xfd^0 +\sqrt{(\xfd^0)^2 - 2g\left( \ell_0\sin(\theta_f)-\xfc^0\right) }\right). }
Once $\xf^0$ is given, $\theta_f$ is the only variable in the above equation. Therefore, there is a local bijection between $T_F$ and $\theta_f$ within the acceptable touch down angle range. Moreover, any feasible touch down state is uniquely determined as $\xf(T_F) =\Xi(\theta_f;\xf^0)$, where the mapping $\Xi$ is given below
\eqn{\Xi(\theta_f;\xf^0)	= \bmat{\xfa^0+\xfb^0 T_F\\ \xfb^0\\ \ell_0\sin(\theta_f)\\ -\sqrt{(\xfd^0)^2 - 2g\left( \ell_0\sin(\theta_f)-\xfc^0\right) }\\ \theta_f }.}

In conclusion, given any initial state $\xf^0$, the corresponding TD manifold is only one dimensional and is exactly parameterized by the touch down angle $\theta_f$.
\end{proof}

This powerful observation of the TD manifold structure allows for the following simple reformulation of~\eqref{eq:f1} using the touch down angle $\theta_f$:
\eq{\label{eq:fsb}\ald{
		\min\limits_{w,\theta_f} \qquad & \int\limits_0^{T_F}  c(\xf,w) \dt + V_\text{S}(\Gamma_{\mathtt{fs}}(\Xi(\theta_f;\xf^0)) )&\\
		\st \qquad & \int_0^{T_F} w \dt =  \theta_f-\xfe^0 \\ 
		& \underline{w} \le w \le \overline{w}, \\
		& \theta_f \in (\frac{\pi}{2},\pi).
}}

The above reformulated optimization problem is essentially a one-dimensional optimization over $\theta_f$ where the cost function is nonlinear and is dependent on the value function of the stance phase problem. Solution to this optimization problem can be easily obtained via numerical optimization methods such as gradient descent. Once the optimal $\theta^*_f$ is computed, the corresponding swing leg angular speed profile is explicitly given by
\eq{\label{eq:sln_fsb} w = \min\left\{\overline{w}, \max\left\{\underline{w}, \frac{\theta_f^* -\xfe^0}{T_F} \right\} \right\}.}

\begin{remark}
    In practice, ground-speed matching is typically considered to reduce impact disturbances in the swing phase control problem through a minimization of the foot speed at touch down. In our SLIP model, the foot is massless, therefore swing leg retraction and velocity reset is omitted here.
\end{remark}

\subsection{Implementation}
\begin{figure}[b] 
	\centering
	\includegraphics[width=0.8\linewidth]{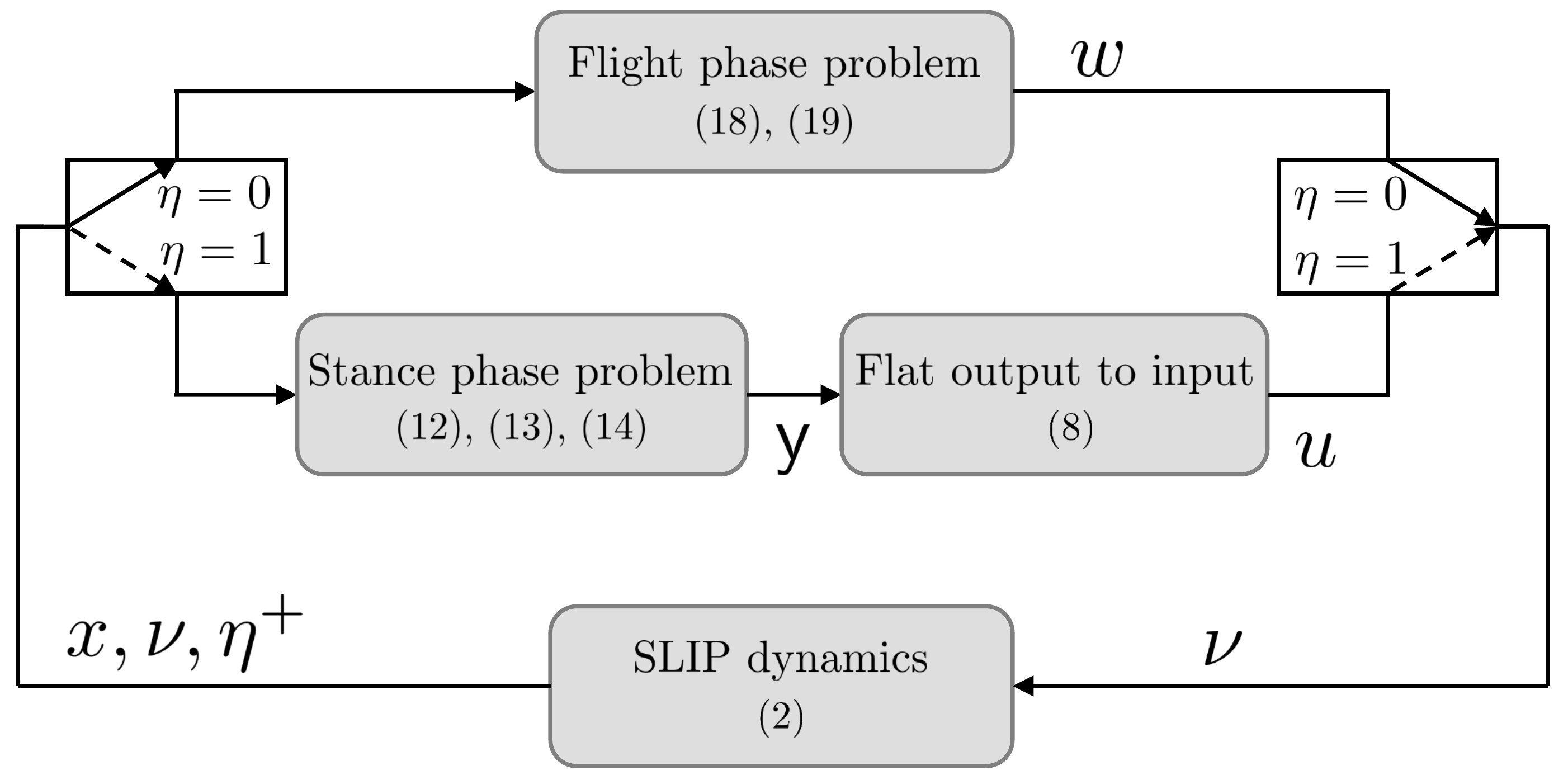} 	
	\caption{\footnotesize Control framework diagram. \label{fig:cdiag} }
\end{figure}
A block diagram summarizing the proposed control framework is given in Fig.~\ref{fig:cdiag}. This control strategy can be implemented in two different manners. The first way of implementing this controller is to solve the optimal control problem only at each touch down or take off event. Once the solution is obtained, it will be implemented in an open-loop fashion until the next touch down or take off event. By virtue of the QP-based solution approach, the proposed controller is able to re-compute the required input signals at a relative high rate. Therefore, the proposed controller can also be implemented in a receding horizon manner. This online re-planning ability is one of the major benefits of the proposed controller which enables active disturbance rejection throughout the SLIP operation. Both implementation strategies will be demonstrated in the following section.

\section{Case Studies}\label{sec:simulation}

In this section, performance of the proposed optimal control strategy for the extended SLIP model is demonstrated through numerical simulations. The SLIP model used in the test has mass of $80\text{kg}$, rest leg length of $1\text{m}$ and spring stiffness of $11\text{kN}\slash\text{m}$, which is modeled after a $50$-th percentile male. The desired take off state in the optimal control problem is selected to be corresponding to the apex state $(1.02\text{m},4.5\text{m}\slash\text{s})$.


\subsection{Comparison with Classical Linearized Deadbeat Controller}
The proposed optimal control strategy on the differentially flat SLIP is first compared with a classical Poincar\'e based once-per-step linearized deadbeat controller~\cite{Wensing13b} on a variable stiffness SLIP. In the comparison, disturbances are modeled as a change of initial apex state, with apex height ranging from $1.01\text{m}$ to $3\text{m}$, and the apex horizontal speed ranging from $1.8\text{m}\slash\text{s}$ to $7.3 \text{m}\slash\text{s}$. 

\begin{figure}[th] 
\centering
\subfigure[\footnotesize Comparison between regions of attraction - disturbance before apex]{%
\includegraphics[width=0.85\linewidth]{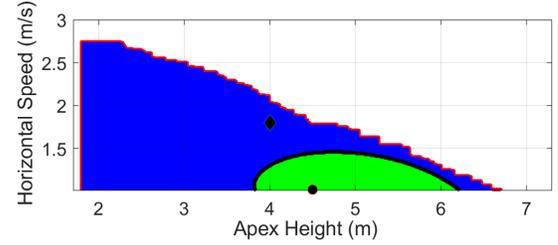}
\label{fig:oc_db_ba}}
\subfigure[\footnotesize Comparison between regions of attraction - disturbance after apex]{%
\includegraphics[width=0.85\linewidth]{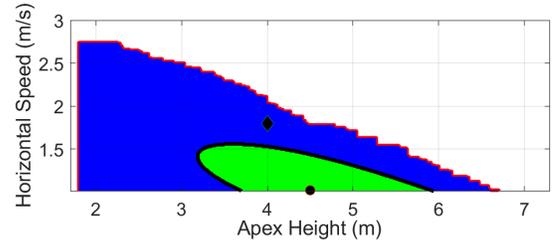} 
\label{fig:oc_db_aa}}
\caption{\footnotesize Regions of Attraction (ROA) of different controllers. RoA for the proposed controller is shown as blue with boundary highlighted in red. RoAs of the deadbeat controller are shown in green, with boundaries highlighted in black. The black circle and diamond show the desired apex state and the initial condition for Fig.~\ref{fig:traj}, respectively.}
\label{fig:oc_db}
\end{figure}
Transient performance is quantified using the accumulated apex error during transients. The proposed controller is updated at each touch down or take off event and the deadbeat controller is updated only at the apex events. The maximum hip torque and the maximum linear actuator displacement for the proposed controller are set to be $400\text{Nm}$~\cite{Zhong2017} and $0.1\text{m}$, while no actuation limits are imposed on the linearized deadbeat controller.  

Fig.~\ref{fig:oc_db} shows the region of attraction (RoA) of different controllers. It can be seen from the comparison that the region of attraction for the proposed controller is significantly larger than the classical linearized deadbeat controller. Moreover, to evaluate the transient performance of the two controllers, we define the accumulated error metric as the sum of the squared apex error from desired apex state over $8$ periods. As shown in Fig.~\ref{fig:tp_oc}, the proposed controller outperforms the linearized deadbeat controller in terms of this transient performance metric in general. Within a small neighborhood of the desired apex state, transient performance of the linearized deadbeat controller in response to disturbances before apex is slightly better than the proposed controller. The main reason of this phenomenon is that the linearized deadbeat controller is only trying to recover the desired apex state without considering the control effort applied to achieve this goal.

\begin{figure}[tb!] 
\centering
\subfigure[\footnotesize Transient performance of the proposed controller]{%
\includegraphics[width=0.85\linewidth]{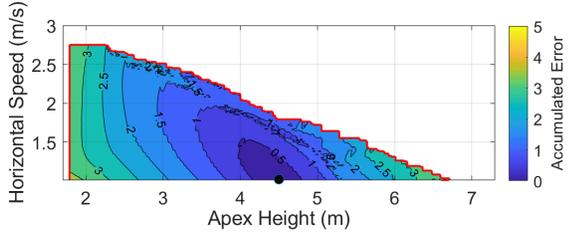}
\label{fig:tp_oc_vx}}
\subfigure[\footnotesize Transient performance of the linearized deadbeat controller with disturbances before apex ]{%
\includegraphics[width=0.85\linewidth]{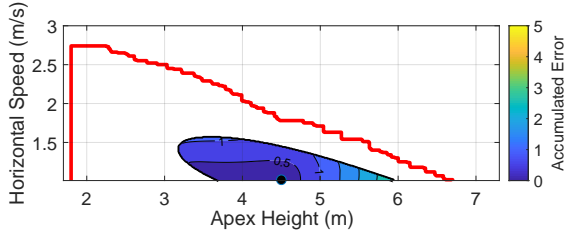} 
\label{fig:tp_oc_hy}}
\subfigure[\footnotesize Transient performance of the linearized deadbeat controller with disturbances after apex]{%
\includegraphics[width=0.85\linewidth]{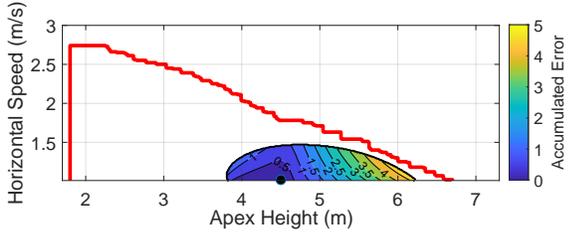} 
\label{fig:tp_oc_hy}}
\caption{\footnotesize Transient performance of different controllers. The boundary of the proposed controller is shown in red. Colormaps show the accumulated apex error during transients. Black circle shows the desired apex state. }
\label{fig:tp_oc}
\end{figure}

In addition to the RoA analysis, we show the CoM trajectory and horizontal speed profiles of a particular initial apex state, selected to be $(1.8\text{m},4\text{m}\slash\text{s})$, shown as a black diamond in Fig.~\ref{fig:oc_db}. It is clear from Fig.~\ref{fig:oc_db} that this apex state lies inside the RoA of the proposed controller while outside RoAs of the linearized deadbeat controller. CoM trajectories and horizontal speed profiles  shown Fig.~\ref{fig:traj}  agree with the RoA analysis.


\begin{figure}[t] 
\centering
\subfigure[\footnotesize CoM trajectories]{%
\includegraphics[width=0.85\linewidth]{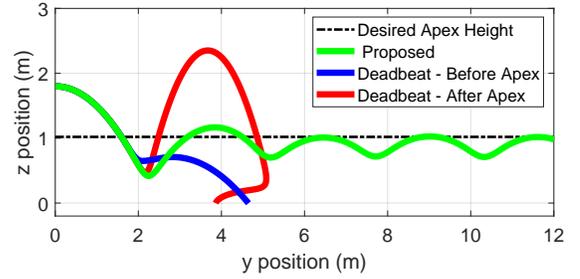}
\label{fig:traj_com}}
\subfigure[\footnotesize Horizontal speed profiles]{%
\includegraphics[width=0.85\linewidth]{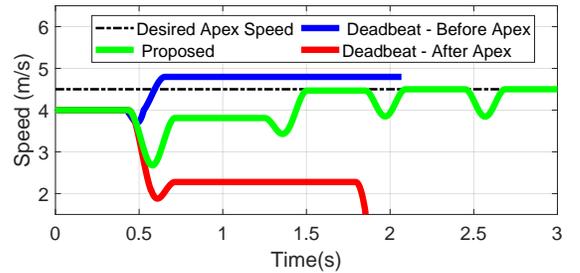} 
\label{fig:traj_velo}}
\caption{\footnotesize Simulation results for a particular apex state $(1.8\text{m},4\text{m}\slash\text{s})$ with different controllers}
\label{fig:traj}
\end{figure}

\subsection{Online Re-planning with Measurement Noise}
The proposed control strategy was also tested with measurement noise. The proposed controller is updated at $20$ Hz, and the computed control is implemented at $1$ kHz. Noise on measurements are modeled as i.i.d. (independent and identically distributed) uniform random variables additive to the true state in horizontal and vertical speeds. As a comparison, we tested the proposed controller under the same scenario but only updated it at each touch down or take off event as well.

\begin{figure}[t] 
\centering
\subfigure[\footnotesize Apex speed]{%
\includegraphics[width=0.85\linewidth]{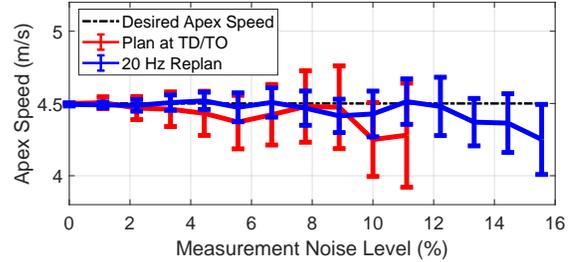}
\label{fig:noise_vx}}
\subfigure[\footnotesize Apex height]{%
\includegraphics[width=0.85\linewidth]{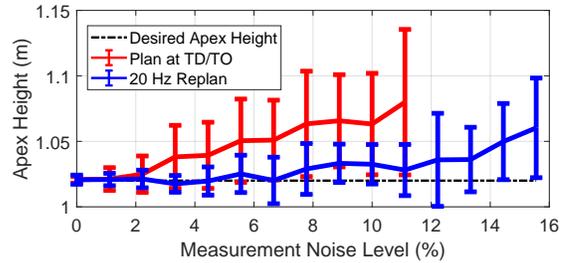} 
\label{fig:noise_hy}}
\caption{\footnotesize Performance of the proposed controller with different levels of measurement noise. Mean and standard deviation of the apex state under different levels of noise are shown. }
\label{fig:noise}
\end{figure}

Performance effects with increasing measurement noise are shown in Fig.~\ref{fig:noise}. With the touch down/take off planning scheme, the proposed controller is able to handle measurement noise up to $\pm 0.5\text{m}\slash\text{s}$ (about $10\%$ of the desired apex speed). With the $20$ Hz re-planning scheme, the proposed controller is able to handle up to $\pm 0.7\text{m}\slash\text{s}$ (about $15\%$) noise level. Despite the fact that measurement noise degrades the controller's performance, the proposed controller still manages to stabilize the SLIP operation at an acceptable level of accuracy. In addition, the online re-planning scheme exhibits a stronger ability in handling measurement noise, yielding smaller steady state error and volatility.

\section{Concluding Remarks and Future Work}\label{sec:concl}
In this paper, we consider an extended spring-loaded inverted pendulum (SLIP) model with leg length and hip torque modulation as well as a kinematic swing leg effect. A hybrid system model of the SLIP dynamics is developed. The dynamics of this extended model during stance are fully actuated, and thereby differentially flat, which has a strong implication regarding controllability. Taking advantage of this powerful feature, a tractable optimal control scheme is developed for rapid trajectory optimization. Jointly with the optimal control scheme, the model offers the capability for active disturbance rejection during both stance and flight. Performance of the control framework is demonstrated via numerical tests and shows practical advantage over existing methods.

The considered SLIP model and its control can potentially serve as a template for planning and control of complex legged robotic systems. In the future, the authors plan to investigate the control architecture of a legged robot in experiments using using the extended SLIP herein as the template model.

\bibliographystyle{IEEEtran}
\bibliography{RA_L_SLIP}

\end{document}